\documentclass[nohyperref]{article}
\usepackage{microtype}
\usepackage{graphicx}
\usepackage{subfigure}
\usepackage{booktabs} 

\usepackage{hyperref}

\usepackage[accepted]{icml2022_TAGML}

\usepackage{amsmath}
\usepackage{amssymb}
\usepackage{mathtools}
\usepackage{amsthm}
\usepackage{multicol}

\newtheorem*{definition*}{Definition}
\newtheorem*{proposition*}{Proposition}

\theoremstyle{plain}

\usepackage[capitalize,noabbrev]{cleveref}

\theoremstyle{plain}
\newtheorem{theorem}{Theorem}[section]
\newtheorem{proposition}[theorem]{Proposition}

\newtheorem{nono-prop}{Proposition}[]

\theoremstyle{definition}
\newtheorem{definition}[theorem]{Definition}

\theoremstyle{remark}

\usepackage[textsize=tiny]{todonotes}

\begin{document}

\twocolumn[
\icmltitle{Invariance-adapted decomposition and Lasso-type contrastive learning}




\begin{icmlauthorlist}
\icmlauthor{Masanori Koyama}{comp}
\icmlauthor{Takeru Miyato}{comp}
\icmlauthor{Kenji Fukumizu}{sch,comp}
\end{icmlauthorlist}

\icmlaffiliation{comp}{Preferred Networks, Tokyo, Japan}
\icmlaffiliation{sch}{The Institute of Statistical Mathematics, Tokyo, Japan}
\icmlcorrespondingauthor{Masanori Koyama}{masomatics@preferred.jp}

\icmlkeywords{Machine Learning, ICML}

\vskip 0.3in
]



\printAffiliationsAndNotice{\icmlEqualContribution} 

\begin{abstract}
Recent years have witnessed the effectiveness of contrastive learning in obtaining the representation of dataset that is useful in interpretation and downstream  tasks.
However, the mechanism that describes this effectiveness have not been thoroughly analyzed, 
and many studies have been conducted to investigate the data structures captured by contrastive learning.
In particular, the recent study of \citet{content_isolate} has shown that contrastive learning is capable of decomposing the data space into the space that is invariant to all augmentations and its complement. 
In this paper, we introduce the notion of invariance-adapted latent space that decomposes the data space into the intersections of the invariant spaces of each augmentation and their complements. This decomposition generalizes the one introduced in \citet{content_isolate}, and describes a structure that is analogous to the frequencies in the harmonic analysis of a group. 
We experimentally show that contrastive learning with lasso-type metric can be used to find an invariance-adapted latent space, thereby suggesting a new potential for the contrastive learning. 
We also investigate when such a latent space can be identified up to mixings within each component.

\end{abstract}

\section{Introduction}




Collectively, \textrm{contrastive learning} refers to a family of representation-learning methods with a mechanism to construct a latent representation in which the members of any \textit{positive pair} with similar semantic information are close to each other while the members of \textit{negative pairs} with different semantic information are far apart \citep{hjelm2018learning, bachman2019amdim, henaff2020cpc, tian2019cmc, chen2020simclr}. 
Recently, numerous variations of contrastive learning \citep{clip, pcl, cont_seg, owod, curl}
 have appeared in literature, providing evidences in support of the contrastive approach in real world applications. 
However, there still seems to be much room left for the investigation of the reason why the contrastive learning is effective in the domains like image-processing.

Recent works in this direction of research include those pertaining to information theoretic interpretation \citep{wang2020uniform} as well as the mechanism by which the contrastive objective uncovers the data generating process and the underlying structure of the dataset in a systematic way \citep{invert}. 
In particular, \citet{content_isolate} have
shown that, under a moderate transitivity assumptions with respect to the actions of the augmentations used in training, the contrastive learning can 
isolate the \textit{content} from the \textit{style}---where the former is defined as the space that is fixed by all augmentations and the latter is defined as the space altered by some augmentation. 

Meanwhile, \citet{orbit_decomp} described the similar philosophy in terms of group theoretic context, claiming that the contrastive learning can decompose the space into inter-orbital direction and intra-orbital direction. 
They even went further to decompose the inter-orbital direction by introducing an auxiliary IRM-type loss.
In a related note, \citet{product_man} 
assume that the dataset of interest has the structure of a product manifold with the assumption that each augmentation family alters only one of the products. 
They train a set of nonlinear projection operators to extract each component of the manifold in a self-supervised way.
When the actions applied to the dataset constitute a group, the ultimate study of inter-orbital direction is the study of group actions and equivariance, because  each orbit has the structure of the group modulo the stabilizer. Several works in past have proposed methods to learn this structure \cite{imagerot, Dangovski, colorful, Zigsaw}, and shown that the knowledage of inter-orbital direction plays an important role in the performance of downstream tasks as well.

In this research, we present a result that suggests a connection between contrastive learning and decomposition that generalizes the decomposition discussed in \citet{content_isolate}.  
In particular, if $\mathcal{T} = \{T_i\}$ is the set of augmentations to be used in the training, we empirically show that a standard contrastive learning with Lasso type distance can be used to decompose the data space into the intersections of 
the invariance spaces of $\mathcal{T}$ and their complements. 
We say that the latent space is invariance-adapted to $\mathcal{T}$ if any invariance space can be represented by a set of coordinates. 
As we will describe in more detail,  this decomposition generalizes the decomposition studied in \citet{content_isolate} and has a connection to harmonic analysis of group.
Moreover, we will show that, in a special case, this decomposition of the space can be block-identified without assuming strict group structure on the set of augmentations. 

\section{Decomposition of invariant spaces with $L_1$ contrastive learning}

\subsection{Invariance-adapted latent space}
We first describe the nature of the invariance-adapted latent space.
Let $\mathcal{X}$ be the space of dataset, and let $\mathcal{T} = \{T_i: \mathcal{X} \to \mathcal{X} \}$ be a set of augmentations. 
\begin{figure}[t!]
    \centering
    \includegraphics[scale=0.3]{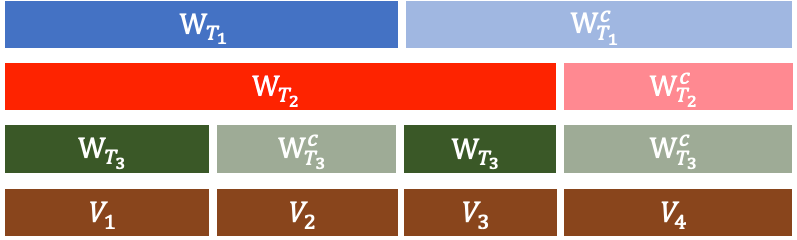}
    \caption{First three rows: visualization of the decomposition of the latent space of $\mathcal{X}$ induced by three different augmentations, $T_1, T_2$ and $T_3$. The space $W_T$ is the invariant space of $T$.  The bottom row: the decomposition of  $\mathcal{X}$ into frequency components $\{V_k\}$ that can generate any intersections of $W_T$s.}
    \label{fig:symmetry_decomp}
\end{figure}
Suppose that $h: \mathcal{X} \to \mathcal{Z}$ is an invertible map from $\mathcal{X}$ to some Euclidean latent space $\mathcal{Z}=\mathbb{R}^n$.
For each augmentation $T\in \mathcal{T}$, we may define the latent $T$-invariant space $W_T \subset \mathbb{R}^n$ to be the set of $w$ for which $T \circ h^{-1}(w) = h^{-1} (w)$\footnote{Note that $T$ fixes any element in $W_T$.  This is different from the popular definition of ``invariant space'' in linear algebra, which requires only $h\circ T \circ h^{-1}( W_T)\subset W_T$ as \textit{a linear subspace}. }. 
We would like to consider a decomposition of the space into the intersections of $W_T$s and their complements $W_T^c$s.
The first three rows in Figure \ref{fig:symmetry_decomp} are the visualization of the decomposition of $\mathcal{X}$ into invariant/non-invariant space for three different choices of augmentations, $T_1$, $T_2$ and $T_3$.
When there are three decompositions as such,
we can consider the decomposition of the latent space into $\{V_k\}$ defined as the minimal intersections of invariant spaces and their complements (fourth row), For example, in Figure \ref{fig:symmetry_decomp}, $V_1 = W_{T_1} \cap W_{T_2} \cap W_{T_3}$, $V_2 = W_{T_1} \cap W_{T_3}^c$,  and so on.

The most important property of $\{V_k\}$ is that, for any arbitrary subset of $\mathcal{T}$, the family $\{V_k\}$ can generate the common invariant space for the subset as well as its complementary space. 
In this work, we would like to consider $h$ such that each $V_k$ can be represented by a set of its coordinates. In other words,
we want $h$ such that, for each $T \in \mathcal{T},$ there is a unique set $A_T$ of coordinates $w_1,\ldots,w_n$ such that
\begin{align}
h_i [T h^{-1}(w) ] = w_i ~~~ \textrm{for all $i \in A_T$,}
\end{align}
where $h_i$ is the $i$-th component of $h$.

The concept we are introducing here is akin to the one used in harmonic analysis of groups \citep{Weintraub, Garsia, clausen}.
To see this connection, consider the discrete Fourier transform $h_{DFT}: \mathbb{C}^4 \to \mathbb{C}^4$, 
which is the change of basis transformation from the standard basis to the Fourier basis: 
$$[1, 1, 1, 1],  ~~ [1, i, -1, -i],  ~~[1, -1, 1, -1], ~~[1, -i, -1, i] $$ 
or $v_k = [1, i^k, i^{2k}, i^{3k} ]$ for $k\in 0:3$.
This basis is special because each member is an eigen-vector of the shift action $\tau([a_1, a_2, a_3, a_4]) = ([a_2, a_3, a_4, a_1])$ and their powers $\tau^2, \tau^3$ and $\tau^0=id$. 
Thus, $h_{DFT}$ is equivalent to the decomposition of $\mathbb{C}^4$ into eigen-spaces of the shift actions.
The left frame in Table\ref{tab:spectrum} summarizes the eigen-spectrum of the Fourier basis. 
From this table we can see that $\tau^2$-invariant space in our sense is given by the span of $\{v_0, v_2\}$, because
$\tau^2(v_1) =\tau^2([1, i, -1, -i]) = [-1, -i , 1, i]= - v_1$, and $\tau^2(v_3)=-v_3$ while $\tau^2(v_2)=v_2$ and $\tau^2(v_0) = v_0$.  
Note that each $v_k$ in the Fourier basis is uniquely characterized by how it ``responds'' to $\mathcal{T}_{\textrm{shift}} = \{ \tau^\ell ; \ell=0:3\}$ because each row in the spectrum table is diferent from each other. 
In the harmonic analysis of groups, each $v_k$ is often called {\em frequency}. 
\begin{table}[]
    \centering
    \begin{tabular}{c | c | c  | c  | c}
          & $\tau^0$ & $\tau$  & $\tau^2$ & $\tau^3$  \\
    \hline
    $v_0$   & $1$ & $1$   & $1$ & $1$ \\ 
    $v_1$   & $1$ & $i$   & $-1$ & $-i$  \\
    $v_2$   & $1$ & $-1$  & $1$  & $-1$  \\
    $v_3$   & $1$ & $-i$  & $-1$ & $i $
    \end{tabular} \hspace*{2mm}
    \begin{tabular}{c | c | c  | c }
          & $T_1$ & $T_2$  & $T_3$   \\
    \hline
    $V_1$   & $1$ & $1$   & $1$  \\ 
    $V_2$   & $1$ & $1$   & $0$  \\
    $V_3$   & $0$ & $1$  & $1$    \\
    $V_4$   & $0$ & $0$  & $0$ 
    \end{tabular}
    \caption{Left : eigen-spectrum of Fourier basis for DFT(3). 
    Right: invariance spectrum of $\{T_1, T_2, T_3\}$ in Figure \ref{fig:symmetry_decomp} }
    \label{tab:spectrum}
\end{table}

Likewise, each $V_k$ in the decomposition of Figure \ref{fig:symmetry_decomp} is uniquely characterized by how it responds to each member of  $\mathcal{T}$. In the right frame in Table\ref{tab:spectrum}, the entry corresponding to $V_k$ and $T_\ell$ is $1$ if $T_\ell$ fixes $V_k$ and $0$ otherwise.  The binary entries in the right frame of \ref{tab:spectrum} is analogous to the eigen-values in the left frame, and each $V_k$ plays a similar role as a {\em frequency} in the DFT we desribed above.
We would like to formalize this idea below.

\begin{definition}
A latent space $h(\mathcal{X})$ for an invertible $h$ is said to be \textit{invariance-adapted to $\mathcal{T}$} 
if for each $T\in\mathcal{T}$ there is a subset of coordinates $A_T\subset\{1\ldots,n\}$ such that
\begin{equation}
h_i (T h^{-1}(w)) = w_i \textrm{~~for all $i \in A_T$ and $w\in\mathcal{Z}$}
\end{equation}
and that, for $j\notin A_T$,  there is some $w'\in\mathcal{Z}$ that satisfies $h_j (T h^{-1}(w')) \neq w'_j$. 
\end{definition}
In other words, when the latent space is invariance-adapted, each $T$-invariant space $W_T$ is defined by $A_T$, and $W_T^c$ is defined by $A_T^c$. 
Also, in an  invariance-adapted latent space, any intersection $\cap_j S_j$ with $S_j\in\{W_T,W_T^c\mid T\in\mathcal{T}\}$ can be expressed by a subset of coordinates.  Therefore, the family of the minimal intersections of $\{A_T, A_T^c\mid T\in\mathcal{T}\}$\footnote{Since we consider only subspaces spanned by subsets of coordinates, we use coordinate indices $\{A_T, A_T^c\mid T\in\mathcal{T}\}$ and spanned subspaces $\{W_T,W_T^c\mid T\in\mathcal{T}\}$ interchangeably.} defines a decomposition of the latent coordinates  $(w_1,\ldots,w_n)$ 
so that any $T$-invariant subspace $W_T$ can be expressed as a union of the minimal intersections.  We call each member of the set of minimal intersections of $\{A_T, A_T^c\mid T\in\mathcal{T}\}$ an {\em invariance-frequency} for the representation $h$, naming it by the analogy to the harmonic analysis.

For example, the coordinate system of $\mathbb{C}^4$ under the Fourier basis in the discussion above is invariance-adapted, where $h$ is the DFT and the three invariance-frequencies are spanned by $\{v_0\},\{v_2\}$, and $\{v_1,v_3\}$ respectively. 
The invariant space of $\tau^2$ is spanned by $\{v_0, v_2\}$, and the invariant space of $\tau$ and $\tau^3$ is spanned by $\{v_0\}$ only. 

Note also that, instead of the $T$-invariant space for an individual $T\in\mathcal{T}$,  we can also consider $\mathcal{T}_a$-invariant space for a subset $\mathcal{T}_a\subset \mathcal{T}$, which is the common space that is invariant to any $T\in \mathcal{T}_a$, and discuss $\{V_j\}$ for a family of subsets $\{\mathcal{T}_a\}_a$.  In the later experiment of image augmentations, $\mathcal{T}$ is composed of the family of rotations $\mathcal{T}_{rot}$ and the family of color-jitterings $\mathcal{T}_{color}$. We would consider the respective invariant space for $\mathcal{T}_{rot}$ and $\mathcal{T}_{color}$, and discuss the decomposition of the latent space according to those three invariant spaces.



\subsection{Sparse method for invariance-adapted representation}

We want to experimentally show that we can use a simple contrastive learning objective to find an invariance-adapted latent space for $\mathcal{X}$. 
As shown in \citet{wang2020uniform}, the contrastive learning based on noise contrastive error (e.g simCLR) can be described as a combination of two losses: (1) the alignment loss that attracts the positive pairs in the latent space and (2) the uniformity loss that encourages the latent variables to be distributed uniformly, thereby preventing the degeneration. 
We use the loss of the same type, except that we replace the cosine distance norm in simCLR with the $L_1$ loss. We thus consider the following objective to train the encoder $h_\theta$ parametrized by $\theta$: 
\begin{align}
\begin{split}
    &L_{\text{align}} + L _{\text{anti-deg}}  \\
    &= \mathbb{E}_{T,X} [\|h_\theta(T(X)) - h_\theta (X)) \|_1/\tau]+  (- H(h_\theta (X))),
    \label{eq:L1loss} 
    \end{split}
\end{align}
where $X$ is the observation input, $T$ is the random augmentations, $\tau$ is the temperature, and $H$ is a function that prevents degeneration, such as Shannon's entropy.
Essentially, this objective function differs from simCLR only in the choice of the metric used to bring the positive pairs together; more precisely, \eqref{eq:L1loss} becomes simCLR when we replace $\|a-b\|_1$ with $a^Tb$ and use a single instance of $T$ instead of a pair of augmentations $T_1$ and $T_2$.
Notice that $L_1$ distance is particularly different from the angular distance in that it is \textrm{not} invariant to the orthogonal transformation, and hence is able to align particular set of dimensions.
The gist of this loss is to maximize the number of dimensions at which the pair $(T(X),X)$ agrees in the latent space in the way of LASSO \citep{tib_lasso}.  
If $\mathcal{T}_1 \subset \mathcal{T}$ has a common invariant space $W_{\mathcal{T}_1}$ and $\mathcal{T}_2 \subset \mathcal{T}$ has $W_{\mathcal{T}_2}$, and if $\mathcal{T}_1$and $\mathcal{T}_2$ contains the identity, we can expect our loss to seek a representation in which $W_{\mathcal{T}_1} \cap W_{\mathcal{T}_2}$ as well as $W_{\mathcal{T}_2}$ and $W_{\mathcal{T}_2}$ are maximized. 

Seeing $L_1$ as an approximation of $L_0$ and noting that a linear invertible map is equivalent to a change of basis, the following proposition provides a partial justification to the objective function \eqref{eq:L1loss} when $\mathcal{T}$ satisfies certain condition:
 \begin{proposition}
Suppose that the augmentations are linear transformations on linear space, and each $T\in \mathcal{T}$ does not mix a member of $W_T^c$ into $W_T$. Let $[v]_{P}$ denote the representation of 
$v$ under the basis $P$.
 If each $W_T$ is a subset of coordinates under some $P^*$,
 then $P^*$ can be identified upto mixing within each invariance-frequency by minimizing
 $$L(P) = \sum_k \max_v \| [T_k v - v]_P \|_0.$$
 \vspace{-0.5cm}
 \label{thm:l0}
 \end{proposition}
For the formal version of statement and its proof, see Appendix.
Unfortunately, the direct optimization of this objective is difficult because
$L_0$ optimziation is in general NP-hard \cite{l0_opt}, let alone its maximum.
Extension of this result to nonlinear case and crafting of the trainable objective function
that is accurately aligned to this result 
is a future work.

\subsection{Uniqueness of representation}

It turns out that, 
if the latent space is invariance-adapted, 
the representation mapping that defines the decomposition can be identified uniquely up to a mixing within each invariance-frequency.
The result below 
is an analogue of the block-identifiability results in \citet{content_isolate}.
or the formal version of this result, see Appendix \ref{sec:appendix_proof}.
\begin{proposition}[Informal]
Let $\mathcal{T}$ be augmentations of $\mathcal{X}$ and 
$h:\mathcal{X}\to \mathcal{Z}$ be an invertible representation. Suppose that $h$ is invariance-adapted with $\{V_j\}$ its corresponding set of invariance-frequencies.  
If $\mathcal{T}$ satisfies a certain transitivity assumption,
then $\{V_j\}$ are block identifiable, that is, if there is any other invariance-adapted representation $\tilde h$ with 
invariance-frequency $\{\tilde V_j\}$, then there exists an invertible map between $V_j$ and $\tilde V_j$ after reordering of $\{\tilde V_j\}$.
\label{thm:identifiability}
\end{proposition}

 \section{Related works}

As we have shown in the previous section and Table \ref{tab:spectrum}, our invariance-adaptation has a close connection with harmonic analysis of groups.  
In particular, if $\mathcal{T}$ is a group, the invariance-adapted basis can be recovered from Fourier basis. 
However,  the setup of our study is fundamentally different from those investigating the way to learn irreducible representations of groups, including \cite{comm_lie} as well as more recent works such as \citet{noncomm_lie, dehmamy2021automatic} because we do not assume the strict group structure for our choice of $\mathcal{T}$.  For example, as we show in the next experimental section, we can define an invariance-adapted space even when the members of $\mathcal{T}$ can not necessarily be simultaneously block-diagonalized.
There is still much room left for the investigation of the structural relation between the data space and the set of augmentations that does not necessraily form a group.


Meanwhile, our invariance-adaptation describes much finer invariance decomposition than those of the relevant works that study the structure of the dataset defined by augmentations.
In particular, \citet{content_isolate} showed that contrastive learning can decompose $\mathcal{X}$ into the space $V_{content}$ that is invariant under the action of all members of $\mathcal{T}$ and its complement $V_{style}$.  
In the example of Figure \ref{fig:symmetry_decomp}, $V_{content}$ is our definition of $V_1$, which is the intersection of $W_{T_1}, W_{T_2}$ and $W_{T_3}$. 
Thus, in the context of Figure \ref{fig:symmetry_decomp}, the structure discussed in \citet{content_isolate} is
only the decomposition of the space into $V_1$ and $\bigoplus_{j \neq 1}  V_j$.
\citet{content_isolate} proves the analogue of Proposition \ref{thm:identifiability} for this decomposition with a probabilistic transitivity condition along with a simply connected manifold assumption.
Meanwhile, \citet{orbit_decomp} discusses a further decomposition of $V_1$, but not its complement.
Also, some works \cite{imagerot, Dangovski, colorful, Zigsaw} explore the feature learning based on the behavior of $\mathcal{T}$, but their main focus is the analysis of the $\mathcal{T}$-variant structure. 



We shall also mention a theoretical result shown by \citet{invert} regarding the $L1$ type contrastive loss. 
\citet{invert} assume that the observation $X$ are generated from latent $Z$ as $h_*^{-1}(Z) = X$ with $Z$ being uniformly distributed over the sphere. 
They show that if the conditional distribution of $P(Z| h_*(T(h_*^{-1}(Z))) )$ is of form $C \exp(-\| Z - h_*T(h_*^{-1}(Z) \|_1 /\tau)$, then the true $h_*^{-1}$ can be identified up to permutation using the L1 type contrastive loss. 
While the loss they use in their analysis is related to our study, we are different in that we make an algebraic assumption about $\mathcal{T}$ and $\mathcal{X}$ that there exists at least one latent space that is invariance-adapted to $\mathcal{T}$, 
while they make a distributional assumption as well as simply connected assumption on the space.
In the next section, we study if we can experimentally learn an invariance-adapted decomposition using our $L_1$ contrastive loss.

\section{Experiments}

\subsection{Linear Case}
The goal in this experiment section is to use our contrasitve loss function \eqref{eq:L1loss} to recover an invariance-adapted latent space from the set of augmented observations in a self-supervised way. 
Our experiment of a linear case is instrumental in describing the effect of \eqref{eq:L1loss}. 
In this experiment, we assume that the dataspace $\mathcal{X}$ is a $6$ dimensional linear subspace embedded in a $12$ dimensional input space via a linear embedding function $h_*^{-1}: \mathcal{Z} = \mathbb{R}^{6} \to \mathcal{X}\subset \mathbb{R}^{12}$.  
This setup is a linear analogue of a $6$ dimensional manifold embedded in a $12$ dimensional ambient space. 
Suppose that there is an augmentation set $\mathcal{T}$ that acts on $\mathcal{X}$. 
Through  $h_*$, this induces a natural action on $\mathcal{Z}$ via $\hat{T}(z) = h_* T h_*^{-1}(z)$ for each $T \in \mathcal{T}$. 
In this experiment, we assume that $h_*$ is invariance-adapted to $\mathcal{T}$. 
That is, we assume that each $\hat{T}$ acts on 
$z \in \mathcal{Z}$ by keeping a subset of coordinates $A_T \subset \{1,...,6 \}$ fixed while mixing its complementary coordinates $A_T^c$. 
In the notation of the previous section, this means that $\mathbb{R}^{A_T} = W_T.$
This situation can be realized when each $\hat{T}$ is a direct sum of identity map and a block map (right panel, Figure \ref{fig:experimental_setup}).
See the left panel in Figure \ref{fig:experimental_setup} for the visualization of this setup. 
The goal of the trainer in this experiment is to learn an invariance adapted latent space from the observation $\{ (T(x) , x) ; T \in \mathcal{T}, x \in \mathcal{X} \}$ only.
Because neither the form of the ground-truth $h_*$ nor the form of $T$ used in the observation set are known to the trainer, the frequency structure in $\hat{T}$ cannot be directly obtained.

\begin{figure}[ht!]
    \includegraphics[scale=0.38]{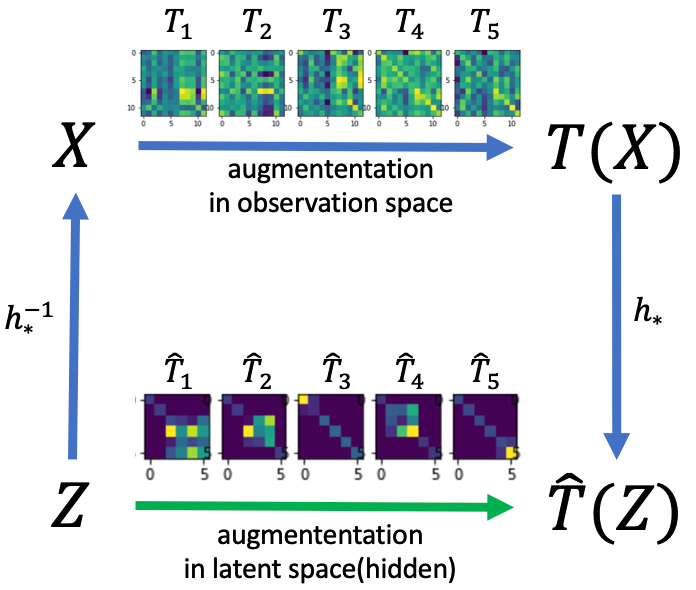}
    \hspace{0.1cm}
    \includegraphics[scale=0.35]{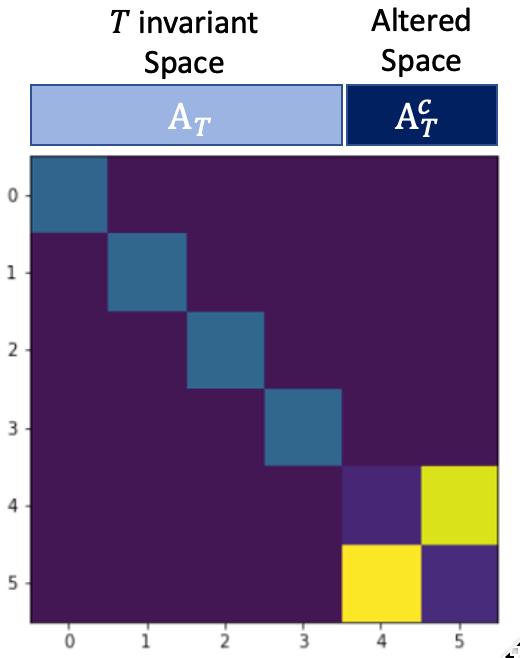}
    \caption{Left panel: Relation between the  augmentation $\hat T=h_* T h_*^{-1}$ in the latent space  and augmentation $T$ in the observation space.
    Right panel: For each $\hat{T}$ , the identity map part corresponds to invariant space $W_T$, and the block part corresponds to $W_T^c$.
    } 
        \label{fig:experimental_setup}
\end{figure}




To generate synthetic samples on $\mathcal{X}$,  we samplped 5000 instances of $Z$ from standard Gaussian distribution, prepared a linear embedding $h_*^{-1}: \mathbb{R}^{6}\to \mathbb{R}^{12}$ as a $\mathbb{R}^{12 \times 6}$ matrix with random entries, and computed each $X \in \mathcal{X}$ as $X= h_*^{-1}(Z)$.
To design $\mathcal{T}$ to which $h^*$ is invariance-adapted, we first created a set of $\hat{T}$ in the form of the right panel in Figure \ref{fig:experimental_setup} by placing a randomly sized block with standard gaussian entries at a random block position.  We then created each transformation $T \in \mathcal{T}$ on the observation domain as $T = h_*^{-1} \hat{T} h_*$  and used them for the contrastive learning on the set of $\mathcal{X}$ produced above. 


To learn an invariance-adapted latent space from the pairs of $(T(X), X)$ alone, we trained  $\hat{h}_\theta$ by our $L_1$ contrastive loss \eqref{eq:L1loss} with the reconstruction-error as the anti-degeneration loss;
more specifically, instead of training the entropy, we simultaneously trained a decoder function 
$g_\theta \approx \hat{h}_\theta^{-1}$ together with $\hat{h}_\theta$ by adding the reconstruction loss $\|g_\theta\circ h_\theta(X) - X\|^2$ with weight $0.05$.
During our training, neither the choices nor the forms of the used $T$s are known to the trainer.

\begin{figure*}
\begin{center}
\hspace*{-2.1cm}\includegraphics[scale=0.43]{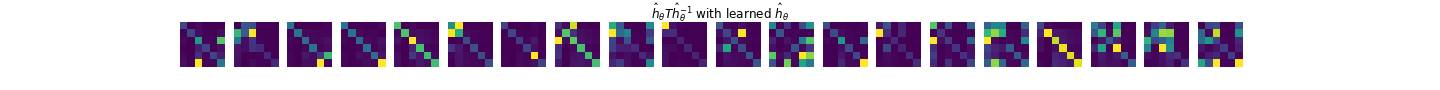} \\
\hspace*{-2.1cm}\includegraphics[scale=0.43]{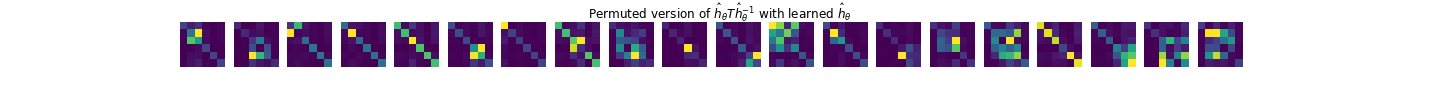} \\
\hspace*{-2.1cm}\includegraphics[scale=0.43]{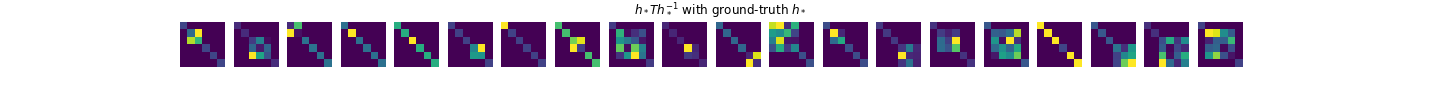}
\vspace{-1cm}
\caption{Visualizations of transformations in the latent space. 
First row: $\hat{h}_\theta T  \hat{h}_\theta^{-1}$ for the same set of $T$s as the third row with $\hat{h}_\theta$ estimated from the $L_1$ contrastive loss. Second row: a coordinate permutation is applied to all the matrices in the first row.  Third row: $\hat{T}= h_* T h_*^{-1}$ for the set of $T$s used in the training of $h_\theta$. In the third row, all diagonal entries not belonging to blocks are $1$.
} 
\label{fig:Linear_aug} 
\end{center}  
\end{figure*}

Figure \ref{fig:Linear_aug} illustrates the result of our training.
The matrices in the third row are the true $h_* T h_*^{-1}$ with different $T$s used in our training.
The matrices in the first row are $\hat{h}_\theta  T  \hat{h}_\theta^{-1}$ obtained with trained $\hat{h}_\theta$, and the matrices in the second row are the result of applying the coordinate permutation to $\hat{h}_\theta  T  \hat{h}_\theta^{-1}$ to best match the positions of the blocks in $h_* T h_*^{-1}$. 
We see that the learned invariant subspace and its complement show perfect match with the ground truth for each $T$. 
Also, all diagonal entries not belonging to the estimated non-trivial blocks are approximately $1$ ($0.946 \pm 0.042$). 
Thus, $\hat{h}_\theta$ identifies all six $V_i$s up to a permutation. 
In other words, we have successfully learned an invariance-adapted basis that can represent the invariant spaces of all augmentations as \textit{sets} of coordinates.



\subsection{Nonlinear case} \label{sec:col_mnist}
\label{gen_inst}
To verify if our method can handle a case in which $h_*$ is nonlinear, we experimented when $\mathcal{X}$ is a stylized MNIST, a modified version of MNIST in which the digits are randomly colored and rotated (Figure \ref{fig:stylized_mnist}.) 
Each image in this dataset is $32 \times 32 \times 3$ dimensional 
\begin{figure}[ht!]
    \centering
    \includegraphics[scale=0.2]{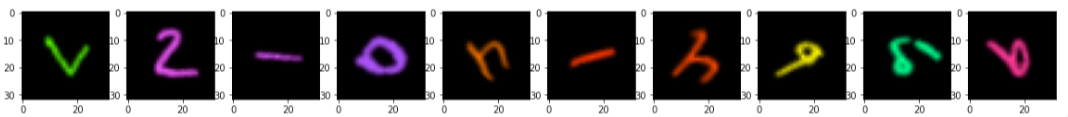}
    \caption{Stylized MNIST dataset}
    \label{fig:stylized_mnist}
\end{figure}
Our goal in this experiment is to find a latent space for $\mathcal{X}$ that is invariance-adapted to $\mathcal{T}$ consisting of color-jitterings and rotations. 

The training in this experiment was more stable when we used a latent space consisting of tensors instead of flat-vectors. 
In particular, instead of setting the latent space to be $\mathbb{R}^{512}$, we used $\mathbb{R}^{32 \times 16}$ as the latent space.
To train the encoder for this tensor latent space, we used group lasso distance \citep{groupLasso} instead of lasso distance;
that is, we chose our encoder $h_\theta$ from the family of maps $\mathbb{R}^{32\times 32 \times 3} \to \mathbb{R}^{32 \times 16}$, and trained it using 
\begin{align}
\begin{split}
    \mathbb{E}_{T,X} [d( h_\theta(T(X)) , h_\theta (X)))/\tau]+  (- H(h_\theta)),
    \label{eq:group_lasso_loss} 
    \end{split}
\end{align}
with $d(z^{(1)}, z^{(2)}) =  \sum_{k \in 0:16} \|z_{k}^{(1)} - z_{k}^{(2)} \|_2$, where $z_{k}^{(i)} \in \mathbb{R}^{32}$ represents the $k$th row of the tensor $z^{(i)}$. 
This tensor representation with group lasso produced better results 
most likely because it has affinity to the structure in which a multiple set of spaces reacts in the same way to a given set of actions (isotypic spaces); if the transformation in the latent space $\mathbb{R}^{a \times s}$ has a matrix representation $\mathbb{R}^{a \times a}$ to be applied from left, the transformation would have a same effect on each column of the tensor. 
This is a common setting in the harmonic analysis of groups \cite{clausen}, and 
such tensor assumption has been used in \citet{topological_defects} as well, 
succeeding to learn the action from a sequence of images.   

Just like in the convention of normalization in contrastive learning, we also normalized each one of $16$ row-vectors.
We trained the encoder with ResNet and set the temperature $\tau$ of the contrastive learning to be $0.001$. Please see Appendices~\ref{sec:arch} and \ref{sec:colormnist_detail} for more details of the architecture and experimental setup.
Our intuition dictates that rotations and color-jitterings are orthogonal in the sense that the latent 
coordinates altered by the member of the former family is disjoint from the latent subspace altered by the latter family. 
Indeed, this turns out to be exactly what we observed in the space learned by our group lasso contrastive learning.
Figure \ref{fig:delta} illustrates the vertically stacked 512 dimensional vectors of flattented $| h_\theta(T(X)) - h_\theta(X) |$ 
for random $T$. 
The first row is produced with
random rotations
$T_{rot}$, and the second row with random color-jitterings $T_{color}$. 
As we can see in the 
the figures, the dimensions that are altered by $T_{rot}$ and $T_{color}$ are exclusive. 
Moreover, we can observe the dimensions that are fixed by both $T_{rot}$ and $T_{color}$ as well. 
We can also see that almost all 512 dimensions are used in the representation (third row).
\begin{figure}
    \centering
    \includegraphics[scale=0.3]{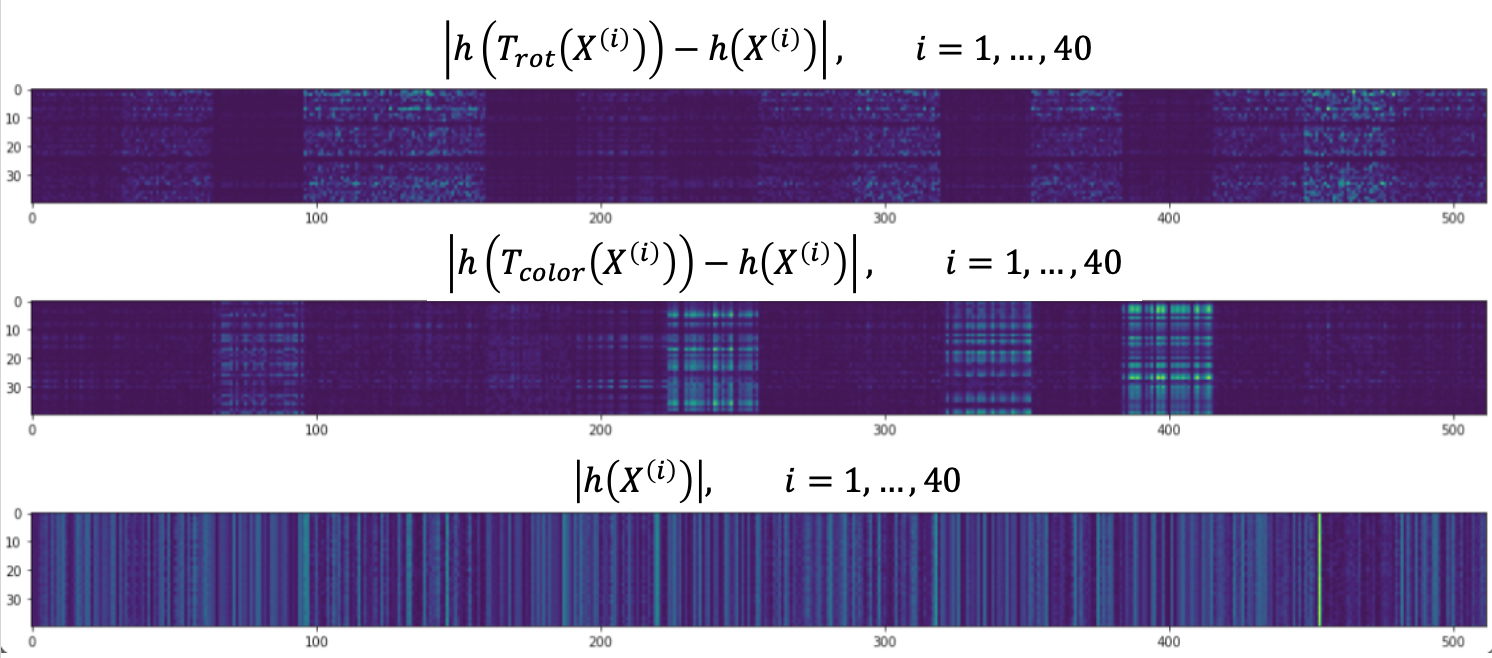}
    \caption{First row: stacked 512 dimensional vectors representing the flattened version of $h_\theta(T_{rot}(x)) - h_\theta(x)$ for random rotations $T_{rot}$.   
    Second row: stacked 512 dimensional vectors representing the flattened versions of $h_\theta(T_{color}(x)) - h_\theta(x)$ for random color rotations $T_{color}$. 
    Third row: stacked 512 dimensional vectors representing the flattened version of $h_\theta(x)$.
    For a better visualization, we took absolute values of each entries in $h_\theta(T(x)) - h_\theta(x)$ so that dark (small value) regions represent the invariant spaces. }
    \label{fig:delta}
\end{figure}
To measure the extent to which the space altered by rotations is complementary to the space altered by the color transformations, we evaluated the following value
\begin{align}
\begin{split}
    &\overline{\mathbb{E}_{T_{rot}, X} } [|h_\theta(T_{rot}(X)) - h_\theta(X)|^2] \times \\
    &~~~\overline{\mathbb{E}_{T_{color}, X}} [|h_\theta(T_{color}(X)) - h_\theta(X)|^2]
\end{split}
\end{align}
Where $\overline{A}$ designates the normalization of the vector $A$. By definition, this evaluates to 0 when the
$T_{rot}$-sensitive space has empty intersection with the $T_{color}$-sensitive space.  In our experiment, simCLR yielded the score of 1.6e-3$\pm$1e-4, while our group-lasso contrastive learning yielded the score of  1.2e-4$\pm$1e-5, numerically validating the complementary decomposition we visually see in Figure \ref{fig:delta}. 
 
Also, to verify whether our decomposed representation holds enough information to predict the important features of our dataset, namely the rotation angle, the color, and the digit shape,  we conducted a linear regression of each one of these features based on the learned representation. 
Table \ref{tab:lcp} summarizes the result of our linear evaluation. For the angle and color, we encoded a pair of images $x_1, x_2$ transformed with random color and random rotation, concatenated the encoded output $z_1, z_2$, and linearly predicted the color difference in RGB and the sine value of the angle difference. 
Because the theme of our study is the representation of the encoder output, we conducted the evaluation on the final output layer. 
As we see in Table \ref{tab:lcp}, our model achieves competitive scores to simCLR for all the features, and predicts the color differences particularly well. 
The raw representation (flattened $32^*32^*3$ vector of image pixels) performed poorly on the test evaluation.
We shall also note that, although the latent represetation obtained from the simCLR-trained encoder (Figure \ref{fig:delta_simCLR}, Appendix) does not feature the decomposition like the one observed in Figure \ref{fig:delta}, it has some ability to predict the features like angle and color, possibly indicating the presence of a hidden invisible structure underneath the representation. 
To visually see if our representation retains enough information of the original images, we also trained 
decoder on the fixed encoder trained from our $L_1$ contrastive objective function. 
As we can see in Figure \ref{fig:img_reconstr} (Appendix)
our representation encodes strong information regarding the orientation, shape and color of the image while featuring the decomposability.


\begin{table}[t]
\small
  \centering
  \scalebox{0.68}{
  \begin{tabular}{l|c|c|c}
  \hline
    Features  &  Ours  &   SimCLR  & Raw representation \\
  \hline
    Digit Accuracy (Logistic)  & $0.5836 \pm 0.0013$ &  $0.6091 \pm 0.0015$ &  $0.5175 \pm 0.0027$ \\ 
    Angle Prediction Error &  $0.0102\pm 0.0002$ & $0.0245 \pm 0.0023$ & $42.4783 \pm 31.6392$   \\
    Color Prediction Error &  $0.0037\pm 0.0001$ & $0.0175 \pm 0.0001$  & $9.4576 \pm 4.5193$  \\
  \hline
 \end{tabular}}
   \caption{Linear evaluation accuracy scores on the learned representations. For Digit, we conducted Linear Logistic Regression. For the prediction of color and angle, we conducted linear regression on the pair of images with random color and rotation to predict the color difference and angle difference. Raw representation fails to produce reasonable outputs outside the training set. }
     \label{tab:lcp}
     
 \vspace{-0.1cm}
\end{table}
\section{Conclusion}
In this work, we introduced a new type of invariance decomposition called invariance-adaptation, and explored a connection between this decomposition and the representation learned by the contrastive learning with Lasso type distance. 
We have shown that, without any auxiliary regularization, the Lasso type contrastive loss has the ability to decompose the dataspace in such a way that the invariant space of any augmentaion transformation $T$ can be represented as a subset of coordinates. 
Our invariance-adaptation not only features the decomposition of the space into entirely insensitive components and its complement, but also carries binary-level information about how each augmentation act on each component of the latent space. 
At the same time, however, there seems to be much room left for learning methods of invariance-adapted latent space. There are many other possibilities for obtaining the coordinates of invariance as well, including the form of the latent space involving tensor structures, for example.
This study should open a door to
a new approach to analyize the contrastive learning and self-supervised learning in general.

\bibliographystyle{icml2022}
\bibliography{references}

\appendix

\section{Proof of Proposition \ref{thm:identifiability}}  
\label{sec:appendix_proof}

Let $\mathcal{T}=\{T\}$ be a set of augmentation transformations on $\mathcal{X}$, and 
$$h: \mathcal{X} \to \mathbb{R}^n, $$
and let 
$$W_T = {w \in \mathbb{R}^n;   h(T h^{-1}(w)) = w}$$
for each $T \in \mathcal{T}$. 
Then we consider the sigma algebra $\Omega$ generated by $\{ W_T; T \in \mathcal{T}\}$ as well as the set $\{V_j\}$ of all its minimal element in the sense of inclusion. Thus, every member of $\Omega$ and hence any intersections/unions of 
$\{ W_T; T \in \mathcal{T}\}$ and their complements can be expressed as unions of $V_j$s. 
These $V_j$ constitutes the set of invariance-frequencies respect to $\mathcal{T}$.
We recall that $h$ is invariance-adapted if, 
for each $T\in\mathcal{T}$ there is a subset of coordinates $A_T\subset\{1\ldots,n\}$ such that
\begin{align}
h_i (T h^{-1}(w)) = w_i \textrm{~~for all $i \in A_T$ and $w\in\mathcal{Z}$}
\end{align}
and that, for $j\notin A_T$,  there is some $w'\in\mathcal{Z}$ that satisfies $h_j (T h^{-1}(w')) \neq w'_j$.


\begin{proposition*}[Formal Version] 
Under the above notations, let $\{V_j\}_{j=1}^m$ be the invariance-adapted subspaces for $\mathcal{T}$, and 
assume that $\mathcal{T}$ acts transitively in the strong sense, that is, 
for any $k$, there exists some $V_k$-fixing $T$ such that $T\circ h^{-1} (v_k, v_{-k}) = h^{-1} (v_k, v_{-k}')$ for any $v_k\in V_k$ and $v_{-k}, v_{-k}'\in V_{-k}:=\bigoplus_{j\neq k}V_j$. Then,  
$\{V_j\}$ are block identifiable, that is, if there is another  $\tilde h$ with 
$\{\tilde V_j\}$ that satisfy the assumptions, then there exists an invertible map between $V_j$ and $\tilde{V}_j$ for each $j$ after appropriate reordering of $\{\tilde{V}_j\}$. 
\end{proposition*}
\begin{proof}
Let $v$ denote an element of 
$\mathbb{R}^n=\bigoplus V_i$, and let $v_{-i}\in \bigoplus_{j\neq i}V_j$ denote a vector with the $i$-th component removed from $v$. 
Note that $\phi = \tilde{h}\circ h^{-1} : \bigoplus_i V_i \to \bigoplus_i \tilde{V}_i$ is an invertible map. 
By the last remark before Proposition, there is a one-to-one correspondence between $\{V_j\}$ and $\{\tilde{V}_j\}$, so we assume w.l.o.g.~that for each $k$, $V_k$ and $\tilde V_k$ have the same set of $T$s to which they are invariant.  
Now, write
$$\phi (v)  = [\phi_k(v), \phi_{-k}(v)]$$ 
where $\phi_k$ is the $\tilde V_k$ component of $\phi$, or $\tilde \pi_k \circ  \phi$ with the projection operator $\tilde \pi_k:\bigoplus_{j=1}^m \tilde V_j \to \tilde V_k$. 
We first show that $\phi_k(v)$ depends only on $v_k = \pi_k (v)$, 
that is, 
$$\phi_k(v_k , v_{-k}') = \phi_k(v_k , v_{-k})$$
for all $v_k\in V_k$ and $v_{-k}, v_{-k}'\in V_{-k}$. 
By the assumption, there exists a $V_k$-fixing $T$ such that $T\circ h^{-1} (v_k, v_{-k}) = h^{-1} (v_k, v_{-k}')$ for all $v_k, v_{-k}$ and $v_{-k}'$. 
Because $T$ is $V_k$-fixing, it is also $\tilde{V}_k$-fixing. 
Thus, 
$$\tilde{\pi}_k \circ  \tilde{h} \circ  T  \circ  \tilde{h}^{-1} = \tilde \pi_k \circ  T  \circ  \tilde h^{-1},$$
or we can say $\tilde{\pi}_k \circ \tilde{h}\circ T = \pi_k \circ \tilde{h}$.
Putting all together, we have 
\begin{align*}
    \tilde{\pi}_k \circ \phi(v_k, v_{-k}) &=  \tilde{\pi}_k \circ\tilde{h}\circ h^{-1}(v_k, v_{-k}) \\
    &= \tilde{\pi}_k\circ \tilde{h}\circ T \circ h^{-1}(v_k, v_{-k}) \\
    &= \tilde{\pi}_k\circ \tilde{h}\circ h^{-1}(v_k, v_{-k}')  \\
    &= \tilde{\pi}_k\circ \phi(v_k, v_{-k}'), 
\end{align*}
which shows that the map $\phi_k$ depends only on $v_k$.
Write $\phi_k(v_k):=\phi_k(v)$, then 
$$\phi(v) = \bigoplus_{k=1}^m \phi_k(v_k).$$ 
Since $\phi$ is invertible, the above relation guarantees that each $\phi_k$ is invertible, and the identifiability follows.


\end{proof}

\section{Proof of Proposition \ref{thm:l0}} 

We introduce several definitions and make several observations before proving the claim.
If $J \subset \{1,2, ..., d\}$ is a subset of coordinates, 
let $M_J \subset \mathbb{R}^{d \times d}$ be the set of matrices defined as 
$$M_J = \{[v_1, ..., v_d]; v_k \in \mathbb{R}^d, v_k = 0 ~~\forall~~ k\in J\}.$$ 
Note that $M_J$ is an \textit{ideal} in the algebra of $\mathbb{R}^{d \times d}$, that is, for any $m \in  \mathbb{R}^{d \times d}$ and $ r \in M_J$, 
$mr \in M_J$. 
Now, also define the set of matrices 
$$H_J = I + M_J = \{I + m; m\in M_J\}$$
By definition,   the kernel of any linear map $w \in H_J$ equals $span(\{e_k; k \in J\})$ almost everywhere, because it is equivalent to requiring  $ w e_k  = w_k = 0$ for all $k \in J$. 
\begin{figure}
    \centering
    \includegraphics[scale= 0.4]{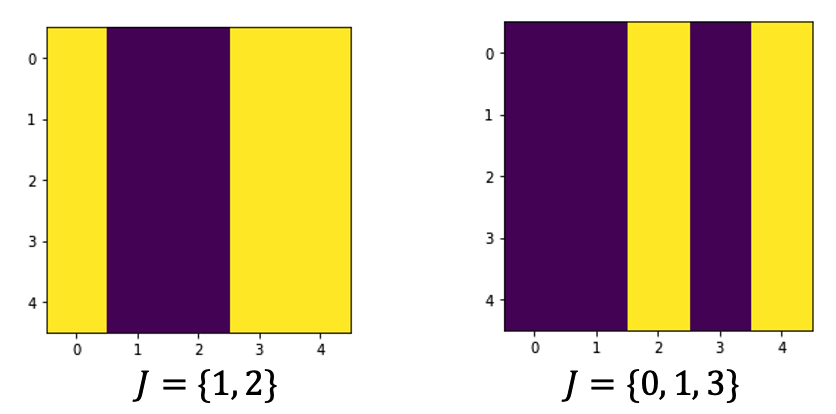}
    \caption{Visualizations of $M_J$.}
    \label{fig:my_label}
\end{figure}

\begin{figure}
    \centering
    \includegraphics[scale= 0.4]{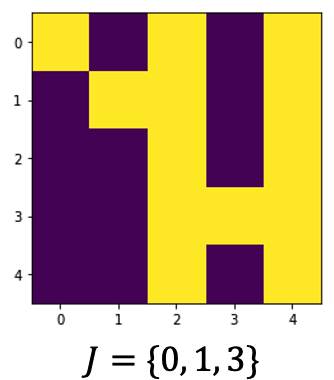}
    \caption{Visualizations of $H_J$}
    \label{fig:my_label}
\end{figure}
Thus, in the linear case, by assuming the existence of invariance-adapted latent space, we are practically assuming the existence of some basis $P^*$ under which $[T]_{P^*} \in H_{A_{T}}$ for each $T \in \mathcal{T}$, where  $A_{T} \subset \{1,2, ..., d\}$. 

 \begin{proposition*}
Suppose that the augmentations are linear transformations on a linear space and let $[v]_{P}$ denote the representation of a vector $v$ under the basis $P$.
 If each $W_T$ is a subset of coordinates under $P^*$ for some $P^*$,
 then $P^*$ can be identified upto mixing within the invariance-frequencies
 by minimizing
 $$L(P) = \sum_{T \in \mathcal{T}} \max_v \| ([T^t]_P - I)[v]_p  \|_0.$$
 \end{proposition*}
 \begin{proof}
 If for all $T \in \mathcal{T}$ we have $[T]_P \in H_{A_T}$ for some $A_T$ and if $[T]_P$ does not mix $W_{T}$ with $W_{T}^c$, then $[T]_P^t \in H_{A_T}$ as well. Thus,
 showing this claim for  $[T]_P^t$ would be equivalent to the original claim.
 This being said, recall that $|v|_0$ counts the number of non-zero entries in $v$. 
 We begin by observing that for any $P$ satisfying
 $[T]_P \in M_{A_T} + I = H_{A_T}$ for each $T$, 
 we have 
 $$\max_v \|[[T]_P^t - I ] [v]_P \|_0 = |A_T^c| $$
 To see this,  note that $\max_v \|[[T]_P^t - I ] [v]_P \|_0 \leq |A_T^c| $ trivially. Also, if $\max_v \|[[T]_P^t - I ] [v]_P \|_0 < |A_T^c|$ then $\textrm{span}(\{ [T]_P^t - I ] [v]_P ; v \in \mathbb{R}^d\} )$ would have the dimension less than $|A_T^c|$. Thus,  this strict inequality  would contradict the maximality of $|A_k|$ which is the kernel dimension of $[T_k - I]_P$ and hence of $[T_k - I]_P^t$.
 Putting 
 $$L_T(P) = max_v \| [ [T]_P^t - I ] [v]_P \|_0, $$ 
we have $\min_P L_T(P) = |A_T^c|$ as well by  the same maximality argument. In other words, $P^*$ in the statement maximizes $L_{T_k}(P)$ for each $T_k$. 

Now, let $\hat{P}$ be another minimizer of $L(P)$,
then $\min_P L_T(\hat{P}) = |A_T^c|$ necessarily by the minimality.
 We claim that this would necessiate that, for each $T$, the change of basis from
 $P^*$ to $\hat{P}$ would map each invariance-frequency to another invariance-frequency of same dimension. 
To see this, let $v^*$ be the vector such that $\| [ [T]_{\hat{P}}^t - I ] [v^*]_{\hat{P}} \|_0 = |A_T^c| $. 
By the maximality of $v$, 
$\| [ [T]_{\hat{P}}^t - I ] [w]_{\hat{P}} \|_0 \leq |A_T^c| $ for all $w$ not in the span of $v$. 
In particular, 
$\| [ [T]_{\hat{P}}^t - I ] e_j \|_0 \leq |A_T^c| $  for all $j$. 
Now, put 
$$\{\ell ; [ [T]_{\hat{P}}^t - I ] e_k ]_\ell \neq 0\} = B_j.$$ 
If $B = \cup B_j$ satisfies $|B| > |A_k^c|$ then there would be some $\{c_j\}$ such that 
\begin{align}
    \left \| [ [T]_{\hat{P}}^t - I ] \left( \sum_j c_j e_j \right)  \right \|_0 = |B| > |A_T^c|
\end{align}
contradicting the maximality of $|A_k^c|$ with respect to $v^*$ above. Therefore $|B| \leq |A_T^c|$ necessarily.
At the same time, if $|B| < |A_T^c|$, this would imply that the number of nonzero rows of $[T]_{\hat{P}}^t$ is greater than $|A_T^c|$ so that $L_T(\hat P) < |A_T^c|$, contradicting the assumed  minimality of $\hat P$.  We therefore have $|B| = |A_T^c|$, and this means that 
the number of nonzero-rows of $L_T(\hat P)$ is exactly $A_T^c$.  

We have thus shown that, for all $T \in \mathcal{T}$, whenever $[T]_{P^*} \in H_{A_T}$, we have 
$[T]_{\hat{P}} \in H_{\hat{A}_T}$ with $|A_T| = |\hat{A}_T|$. because $[T-I]_{P^*} \to  [T-I]_{\hat{P}}$ is a change of basis transformation via a linear invertible map, this change of basis  
 maps $\mathbb{R}^{A_T} \to \mathbb{R}^{A_{\hat{A}_T}}$ invertibly and hence their intersections as well. 
 \end{proof}

\section{Additional Figures}
\begin{figure}[ht!]
    \centering
    \includegraphics[scale=0.2]{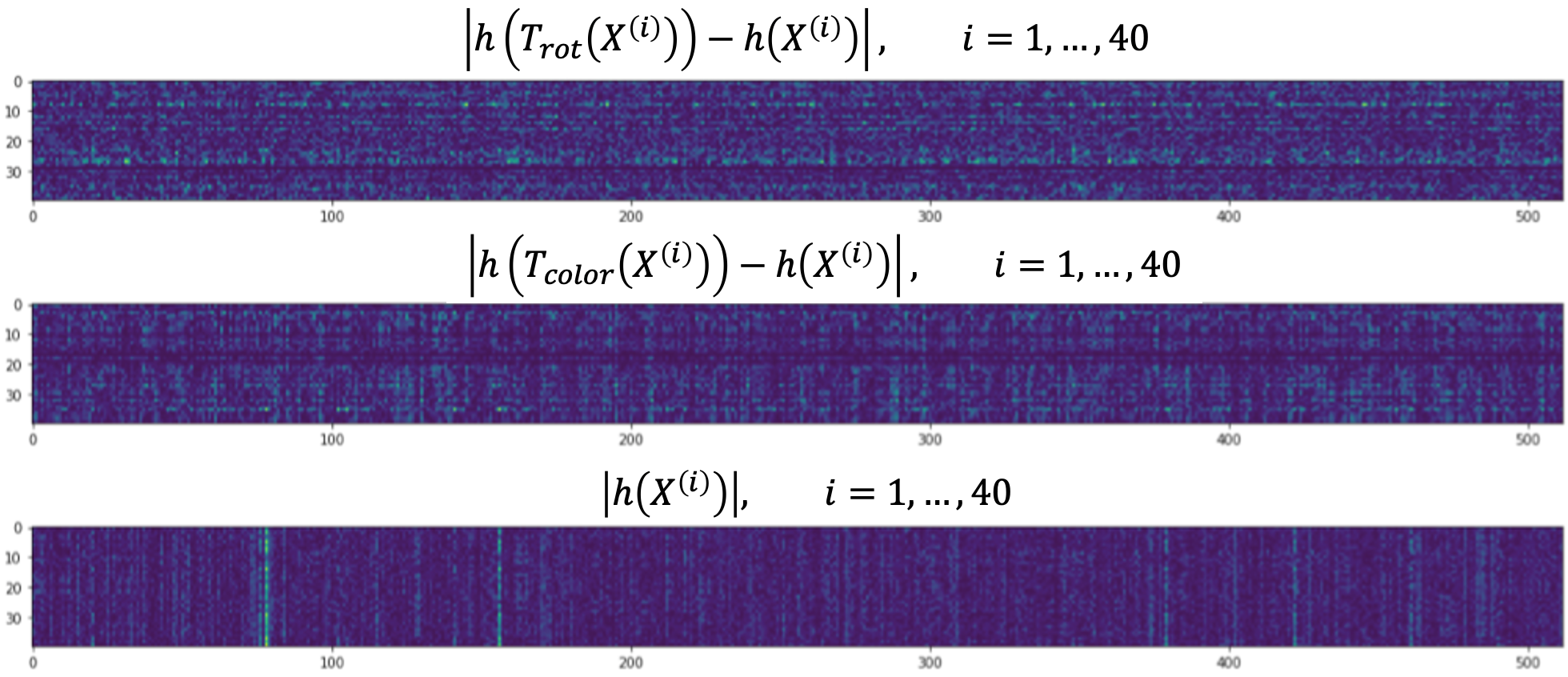}
    \caption{Visualizations of SimCLR-learned representation. First row: stacked 512 dimensional vectors representing the flattened $| h_\theta(T_{rot}(X)) - h_\theta(X) | $ for random rotations $T_{rot}$. Second row: stacked 512 dimensional vectors representing $|h_\theta(T_{color}(X)) - h_\theta(X)| $ for random color rotations $T_{color}$. 
    Third row:  Second row: stacked 512 dimensional vectors representing $|h_\theta(X)|$. No apparent structure is visible in simCLR representation.}
    \label{fig:delta_simCLR}
\end{figure}

\begin{figure}[ht!]
    \centering
    \includegraphics[scale=0.3]{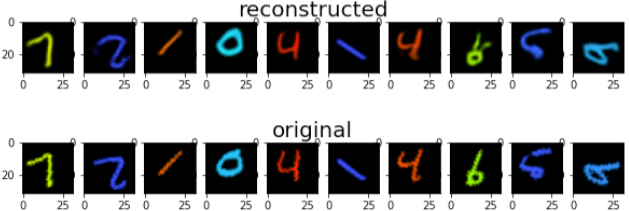}
    \caption{We trained an decoder for the fixed encoder trained with our lasso contrastive loss. We see that our encoder with decomposed feature dimensions retains much information about the digit shape, orientation and color.} 
        \label{fig:img_reconstr}
\end{figure}

\begin{figure}
    \centering
    \includegraphics[scale = 0.3]{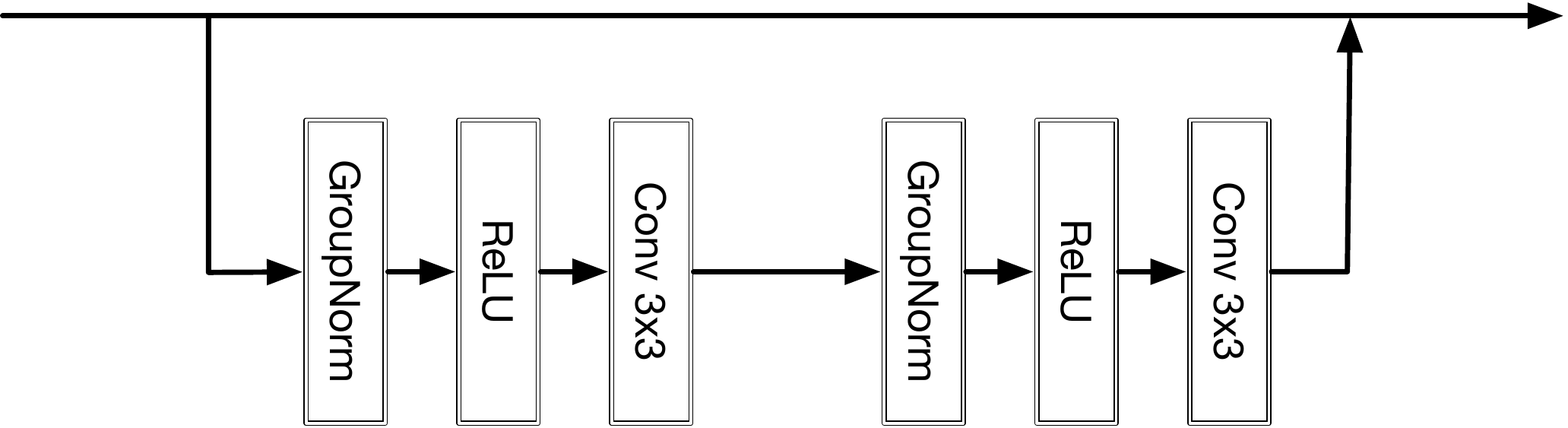}
    \caption{The detail of the ResBlock architecture. For the upsampling and downsampling, we followed the same procedure in ~\cite{miyato2018cgans}. For downsampling, we replaced the identical mapping with 1x1 convolution followed by downsampling layer (mean avegrage pooling). For upsampling, we also replaced the identical mapping with nearest-neighbor upsampling followed by 1x1 convolution. The number of groups for the group normalization layer was set to 32. Weight standarization~\citep{weightstandardization} was applied to each 3x3 convolution layer.}
    \label{fig:resblock}
\end{figure}

\section{The architecture used in Styled MNIST experiment }\label{sec:arch}
 In the experiment in Sec.\ref{gen_inst}, we adopted ResNet~\citep{he2016deep} for the encoder and decoder architecture. We used ReLU function~\citep{nair2010rectified, glorot2011deep, maas2013rectifier} for each activation function and the group normalization~\citep{wu2018group} for the normalization layer. The details of the architecture is found in Figures~\ref{fig:autoenc} and \ref{fig:resblock}.
\begin{figure}[ht]
\center
\includegraphics[scale=0.2]{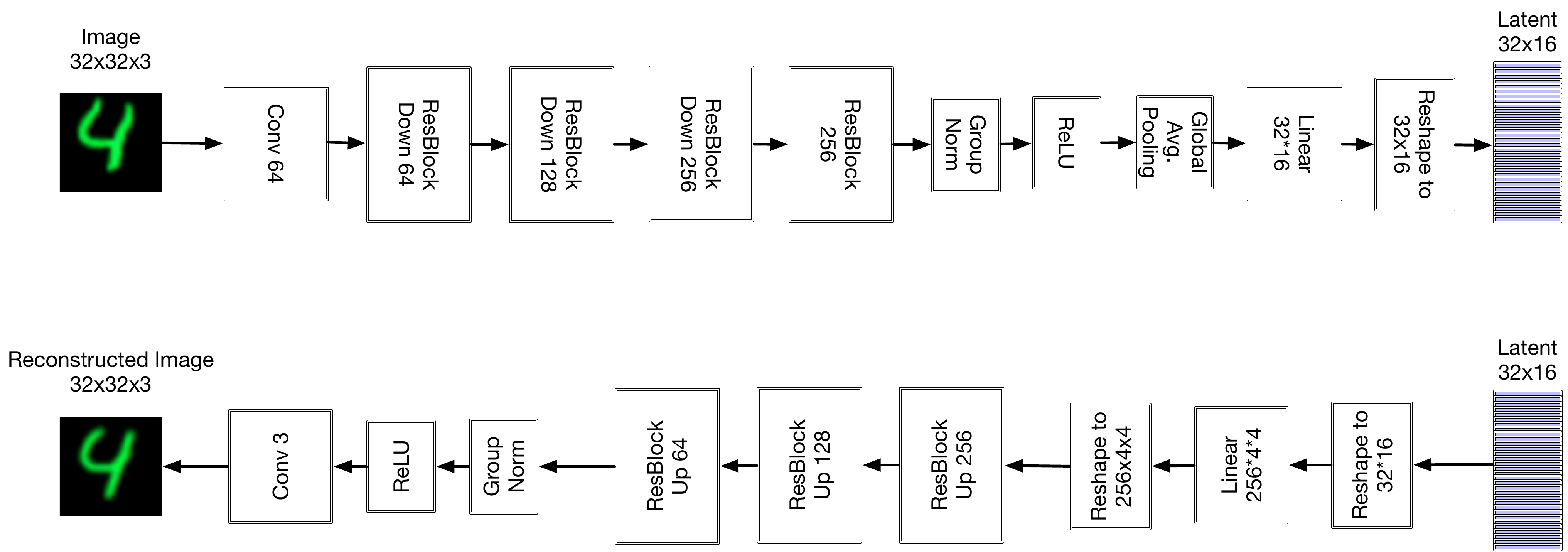}
\caption{The encoder~(top) and decoder (bottom) architecture we used in our experiments. The detail of the ResBlock architecture is described in Fig.\ref{fig:resblock}. The kernel size of all of the convolution layers is set to 3x3 except the layers replacing the skip connection in the ResBlocks, which is set to 1x1. Each number following each layer name indicates the output dimension for the linear layer and output channels for the convolution layer.} 
\label{fig:autoenc}
\end{figure}

\section{More details of the experiment in section \ref{sec:col_mnist}}
\label{sec:colormnist_detail}
In this experiment, we created our stylized mnist dataset $\mathcal{X}$ by resizing each member of mnist dataset to $32 \times 32$, applying a random rotation with uniform angle in the range $[0, 2\pi)$ and coloring it with $hsv(\alpha, 1, 1)$ with $\alpha$ sampled uniformly over $[0,1]$. 
As an augmentation, we used torchvision transformation of random rotation over $[0, \pi)$ and random color-jittering with hue parameter $0.5$ and brightness $0.5$. 
We trained the encoder with the architecture described in Figure \ref{fig:autoenc} based on the SimCLR type contrastive objective
\begin{align}
    \mathbb{E}_{T, X}\left[ log \left( \frac{\exp(-d(h_\theta(T(X)), h_\theta(X))/\tau)}{ \mathbb{E}_{X'}[\exp(-d(h_\theta(T(X)), h_\theta(X'))/\tau)] }\right) \right] 
\end{align}
with $d(z^{(1)}, z^{(2)}) =  \sum_{k \in 0:16} \|z_{k}^{(1)} - z_{k}^{(2)} \|_2$, where $z_{k}^{(i)} \in \mathbb{R}^{32}$ represents the $k$th row of the tensor $z^{(i)}$.
This is an approximation of the objecitve \eqref{eq:group_lasso_loss} 
since the denominator of SimCLR approximates $H(h_\theta(T(X)))$ \cite{wang2020uniform}. 
We set our $\tau = 0.001$, and  trained the model over 250 epochs using Adam with default torch setting.

For the evaluation of the performance of downstream tasks,
we trained linear regression models whose inputs are $512$ dimensional vectors produced by flattenning the $32 \times 16$ dimensional latent tensors.
For digit classification, we used linear logistic regression.
For the angle prediction, we first created a labeled dataset consisting of triplet 
$$( T_{rot(\eta_1)}(x) ,  T_{rot(\eta_2)}(x),  \sin(\eta_1 - \eta_2) )$$
where $T_{rot(\eta)}$ is a clockwise rotation of angle $\eta$, 
and made the model predict $\sin(\eta_1 - \eta_2)$ from
$[h_\theta(T_{rot(\eta_1)}(x)),  h_\theta(T_{rot(\eta_2)}(x))]$.
For the color prediction, we made the model predict the color of $x$ from $h_\theta(x)$. 
For all the training of linear models, we used scipy package with default setting.


\end{document}